\documentclass[10pt, a4paper]{article}

\usepackage{lrec-coling2024} % this is the new style

\usepackage{subcaption}
\usepackage{graphicx}

\usepackage{amsfonts}
\usepackage{amsthm}
\newtheorem*{theorem}{Theorem}

\title{What Happens to a Dataset Transformed by a Projection-based Concept Removal Method?}

\name{Richard Johansson} 

\address{Department of Computer Science and Engineering \\
         University of Gothenburg and Chalmers University of Technology\\
         richard.johansson@gu.se\\}

\abstract{
We investigate the behavior of methods that use linear projections to
remove information about a concept from a language representation, and
we consider the question of what happens to a dataset transformed by
such a method.
A theoretical analysis and experiments on real-world and synthetic data show that these methods inject strong statistical dependencies into the transformed datasets.
After applying such a method, the representation space is highly structured: in the transformed space, an instance tends to be located near instances of the opposite label. As a consequence, the original labeling can in some cases be reconstructed by applying an anti-clustering method.
\\ \newline  }

\begin{document}

\maketitleabstract

\section{Introduction}

%Introduction. Regular citation \cite{ravfogel2020}.
%Testing newcite: \newcite{ravfogel2020}.

While most research in representation learning for NLP focuses on what
information \emph{is} encoded in a representation, in several
scenarios it is important to be able to control what \emph{is not}
encoded.
Most of the discussion in the NLP community has focused on demographic
attributes \cite{bolukbasi2016}.
Another area of application is in domain adaptation: intuitively, if 
representations are uninformative about which domain a data point was
sampled from, learned predictors based on those representations should
generalize more robustly across domains 
\cite{ganin2015}.

% TODO: what does it mean to "not be encoded"?

% 2. discuss various types of methods
% in particular, projection-based methods are mathematically much nicer than general adversarial methods
% is this too much of a related work discussion?
A wide range of methods have been developed to learn a transformation
of representations to try to 
enforce invariance with respect to a given concept while training machine learning models.
% "early approaches" också Bolukbasi
While early approaches were mainly based on adversarial
training \cite{ganin2015}, a number of recent methods have
used \emph{linear projections} for concept removal.
For instance, the Iterative Nullspace
Projection (INLP) method \cite{ravfogel2020} 
projects into a nullspace orthogonal to a set of linear models trained
to predict the concept we wish to remove.
Compared to adversarial methods, projection-based methods are
mathematically more stable, more efficient, easier to implement,
and have performed better in comparative evaluations.

% 3. create a niche -- how?
% high-dimensional intuition?
% Kanske jag kan formulera detta mer kompakt, och resonera mer
% detaljerat i Analysis?
We focus here on the use case where we want to transform the
representations in a given \emph{dataset} to make a given concept impossible
to recover.
For instance, for a given set of word embeddings, we may want
to create a transformed set where the gender variable is
impossible to predict, and then distribute this transformed set to the public.
Other application areas include those where we want to carry out a statistical analysis on a dataset and ensure that some concept does not influence the analysis; for instance, \newcite{daoud2022} discuss the challenges to text-based causal inference methods caused by \emph{treatment leakage}, that is when texts are contaminated by the treatment variable. % and we may want to remove the traces of the treatment variable be
Naively, a user could think that the direct application of projection-based
concept removal would lead to a processed dataset
%where the concept is statistically independent of the
%representations.\footnote{This is of course an abuse of terminology
%with a finite sample.}
resembling one sampled from a distribution where the concept is
statistically independent of the representations: projection-based methods are claimed to ``remove the linear information'' about the undesired concept.
%While this is not claimed explicitly by \newcite{ravfogel2020},
%informal statements such as ``remove the linear information'' are
%used in the paper.
To what extent is it actually true that the
information about the concept is removed from the dataset?

% 4. occupy the niche
% in this paper, we investigate the behavior of these methods.
In this paper, we investigate properties of datasets where
a projection-based concept removal method has been applied to a dataset as a whole.
%
%That is, we are going to focus on properties of the datasets, not on machine learning models trained on them.
%
The main takeaway is that the transformed representation space is
highly structured: the %i.i.d. assumption 
assumption of independent and identically distributed (i.i.d.) instances
does not hold
after applying the method.
%
%While we tyically
Instead of resulting in statistical independence between %columns
the representation and the concept, we show that the concept is reflected in
dependencies between rows (instances) in the transformed datasets.
This injected row-wise dependence is present even in cases where there was no statistical dependence between the representations and the concept in the first place.
We discuss the technical reasons for
why this is the case, and then carry out a series of experiments to
investigate the consequences of this observation.
Our findings include the following:
\begin{itemize}\addtolength{\itemsep}{-0.5\baselineskip}
\item Cross-validation accuracies for predicting the removed concept in transformed datasets are lower than chance.
\item The distribution of prediction probabilities for cross-validated classifiers trained on projected representations are significantly different from
those trained on i.i.d. data.
\item In the transformed dataset, instances tend to be near
those of the opposite category.
\item The original labels can sometimes be decoded from the
transformed dataset by applying anti-clustering methods.
\end{itemize}
We finally discuss the implications of these findings for
practitioners using projection-based concept removal methods to process datasets.
%\todo{soften conclusions? i don't want to diss this type of methods}

\section{Concept Removal Methods}

% discuss connections to PCA, partial least squares

Most early work on 
methods that remove a concept was based on adversarial methods
originally developed for learning domain-invariant
representations \cite{ganin2015}. Adversarial methods have, among other use cases, been
applied for the removal of demographic attributes \cite{raff2018,li2018,barrett2019}.
%More recent approaches to learning invariant representations stress
%the importance of having an explicit model of the causal structure of
%the underlying data-generating process \cite{veitch2021}.

Adversarial training is often unstable in practice and can be
difficult to train because of the minimax objective.
A mathematically more straightforward approach is to use a
linear \emph{projection}, originally introduced by \newcite{xu2017}.
%
%The goal in INLP and related methods is to find representations that
%are \emph{linearly guarded}, meaning that a linear classifier trained
%on this dataset should have a chance-level accuracy.
%
Although recent progress in NLP
highlights the importance of representations computed using nonlinear
functions, it seems that in practice linear projections work well for concept
removal even when nonlinear predictors are used.
\newcite{ravfogel2020} proposed the Iterative Nullspace Projection (INLP)
method that is one of the methods we consider in this paper.
INLP iteratively trains a linear classifier
to predict the concept, and then projects into the subspace orthogonal
to the normal vector of the classifier's separating hyperplane.

More recently, a range of methods intended to improve over INLP have
been developed.
\newcite{ravfogel2022} unified the projection-based and
adversarial families, and presented a method called R-LACE that finds
a projection adversarially.
%; by restricting to linear projections,
%suitable convex relaxations can be devised.
\newcite{belrose2023}
presented a theoretical formalization of conditions for linear
guardedness and an approach to finding optimal projections.
%
% TODO MP
\newcite{haghighatkhah2022} described two variants of \emph{mean projection} (MP), where the difference vector between the class centroids
defines the projection, and they argued that this method is more
effective and less intrusive than INLP.
%\footnote{Mean
%projection is equivalent to INLP 

\section{Theoretical Analysis}
\label{sec:structure}

We investigate the structure of datasets where projection-based
concept removal methods have been applied,
and we are interested in how such datasets differ from 
%The main observation is that that they
%have a structure that is different from what
a normal dataset where the representation $X$ is
statistically independent of the concept $Y$.
%
%Instead, the space is structured ``adversarially'' so that individual
%instances $x_i$ are more likely to be outliers in their neighborhood
%than what would have been expected in random data.
%
%This means that the projection \emph{encodes} information
%about the concept in the new representation.
%
%
%We will demonstrate this effect
%empirically in \S\ref{sec:experiments}, but let us first give
%conceptual intuitions for why this is the case.
%
In this section, we take an analytical perspective and explain theoretically
the structured arrangement of data points
in the transformed space. 
%
%In the next section, we also observe similar tendencies
%empirically with other types of projections.
%projections.
In the next section, we show the results of empirical investigations
complementing the theoretical analysis.
%The theoretical analysis is lim

Our main result shows that instances after projection have an
adversarial arrangement where each instance tends to be located close
to those of the opposite label.
For simplicity of analysis, we limit this analysis to MP
\cite{haghighatkhah2022}, which is equivalent to 
applying INLP with a nearest centroid classifier to find the
projection vector.
A full analysis of the general case is beyond the scope of
this work because it depends on the data-generating distribution as
well as the choice of method used to define the projections.

% TODO:
% en instans kan vara identisk med LOO-centroiden!
\begin{theorem}
Let $X \in \mathbb{R}^{m,n}$ be a feature matrix and $Y \in \{0,1\}^m$
the class labels.
%We further assume that not all rows in $X$ are
%identical.
MP is then applied to $X$ with respect to $Y$ and we refer to the result as
$X_{\mathrm{MP}}$.
%
%The leave-one-out cross-validation accuracy of a nearest-centroid
%classifier applied to $X_{\mathrm{MP}}$, $Y$ is equal to 0.
%In a leave-one-out cross-validation, a nearest centroid classifier
We carry out a leave-one-out cross-validation in the transformed
dataset where we set a single
instance $x_i, y_i$ aside and train a nearest-centroid classifier on
the remaining data.
%In this case, we have one of two situations:
%\begin{description}\addtolength{\itemsep}{-0.5\baselineskip}
%\item[Case 1] $x_i$ is misclassified, or
%\item[Case 2] $x_i$ is on the classifier's decision boundary.
%\end{description}
In this case, $x_i$ cannot be classified with a positive margin by
this classifier: that is, $x_i$ is either misclassified or exactly on
the classifier's decision boundary.
\end{theorem}
\begin{proof}
In MP, the vector used to define the projection is equal to the
difference between the class centroids in the original dataset $X, Y$.
This means that in the projected dataset $X_{\mathrm{MP}}$, the
two class centroids $c_0$ and $c_1$ are identical.
%When expressed as a linear classifier, its normal vector
%is equal to the difference between the centroids, which means that 
%
%in the projected space, the centroids $c_0$ and $c_1$ become identical. % if we recompute them.
%
%In a leave-one-out cross-validation, if we want to classify an
%instance $x_i$ whose label is $y_i$, the centroid
%$c'_{y_i}$ excluding $x_i$ will be shifted from $c_{y_i}$ in the
%direction away from $x_i$, so a nearest-centroid classifier is
%guaranteed to misclassify $x_i$.
%
For ease of exposition, assume that $y_i=0$ and
that the number of instances in class 0 is $n_0 > 1$.
%and $n_1 > 0$, respectively.
%
The centroids of the leave-one-out classifier are $c'_0 = \frac{n_0
  c_0 - x_i}{n_0-1}$ and $c'_1 = c_1$.
Now, we have one of two cases. If $x_i$ is identical to $c_0$, then
$c'_0=c_0=c_1$ so this instance is exactly on the classifier's
decision boundary.
Otherwise, the removal of $x_i$ shifts the center of mass of $c'_0$
in the direction away from $x_i$, so $x_i$ is closer to $c'_1$ than to $c'_0$
and the instance is misclassified.
\end{proof}

This shows that the transformed dataset is fundamentally different
from one where the representation and the label are statistically
independent: if that were the case, the probability of
an instance being classified correctly should correspond to the prior
probability of its class.

It should be stressed that the first case (that the instance is
exactly on the decision boundary) happens only in the theoretical case
that the instance coincides exactly with its class centroid. In
reality, the probability of this to occur is small in practice and
we have never observed it experimentally: all LOO cross-validation
accuracies we have seen with MP have been exactly 0.
We imagine that the corner case may occur more frequently in
datasets where many instances are identical.

%Instead, we
%suggest the following intuition: in a high-dimensional setting, when
%we have independence between $X$ and $Y$ and when instances are sampled
%i.i.d., it is highly probable that some direction of $X$ will be
%correlated with $Y$ in this sample simply by chance. An
%``adversarial'' arrangement of instances would be necessary in order
%to ensure that all directions are uninformative.
%
%
%
%However, in high-dimensional settings in particular, 
%However, consider the following intuition. 
%We hypothesize that projections
%
%Adversarial methods such as R-LACE select a projection in order to
%maximize the training loss; for this loss to be maximized, instances
%should have a non-i.i.d. structure.%
%

% TODO discuss the choice of classifier model to train the linear separator?
% TODO for a SVM, only the SVs are involved in computing w.

% intuition: a set of data points copied
% any direction will be uninformative

\section{Experiments}
\label{sec:experiments}

In the following, we carry out a set of experiments illustrating the
consequences of the observations described in \S\ref{sec:structure}.
%
%For reasons of computational efficiency,
We focus on INLP here to complement the theoretical analysis
of MP in the previous section.
Tentative experiments indicate that the tendencies are
similar when applying R-LACE, but we do not investigate this
algorithm thoroughly because it is computationally more demanding.

\subsection{Datasets}

We carry out the experiments on synthetic and real-world natural
language datasets.
The synthetic data was used for investigating how INLP behaves when
applied to data that does \emph{not} contain any signal representing
the concept.
For the $X$ variables, we generated instances from a standard isotropic
multivariate Gaussian. The labels $Y$ were balanced.
%with positive and negative instances distributed randomly over the dataset.
%drawn from a
%coin-toss distribution
%uniform Bernoulli unrelated to $X$.

For experiments using natural data, we used six domains from the sentiment
classification corpus collected by \newcite{blitzer2007}. This
corpus associates each document with a positive or negative polarity
label; the label distribution is balanced. We did not use the domain information. % here. %, as well as a
% product type label.
%
In the experiments, the representations $X$ were tfidf-weighted bag of
words (BoW) and the output of a BERT model \cite{devlin2019} at the \texttt{[CLS]} token.
%
%We experimented with apply
%The $Y$ labels are again balanced.
%We only considered the prediction of the polarity in the experiments
%below; results for predicting the product type were similar and are
%omitted for brevity.
%and of the product type
%(e.g \emph{books} vs. \emph{music}).
%
We considered different values of the number $n$ of instances, and set
the number $d$ of features to $2^{10}$ for random and BoW
representations. (The overall picture is similar with other values of
$d$.) For BERT, the number of dimensions is 768.

For all datasets, we applied the INLP algorithm. As described above,
the algorithm iteratively trains a linear classifier, and we used a
$L_2$-regularized logistic regression model for this purpose.
We ran INLP for several iterations and the result after each
iteration will be considered in the experiments.

\subsection{Prediction Accuracy}
\label{ss:cv}

We applied INLP to the datasets and computed 32-fold % the leave-one-out
cross-validation accuracy scores for predicting the removed
concept. In all experiments, we used a
$L_2$-regularized logistic regression model ($C=1$) applied to the 
$L_2$-normalized output of the INLP algorithm.
Figure~\ref{fig:cv_acc} shows the accuracies over the INLP iterations
for the BoW and BERT representations. %(The random 
%. We varied the number of columns (representation dimensionality) and rows (instances).
%(We exclude random representations since this plot is similar to
%the BoW plot.)%
We show the results for different sizes $n$ of the dataset.

\begin{figure}[htbp]
\begin{center}
\begin{subfigure}[b]{0.23\textwidth}
\includegraphics[width=0.90\textwidth]{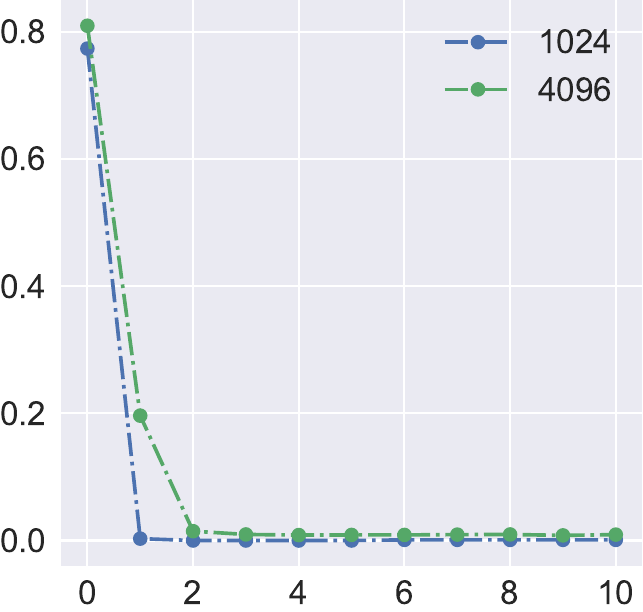}
\caption{Bag of words ($d=2^{10}$).}
\end{subfigure}
\begin{subfigure}[b]{0.23\textwidth}
\includegraphics[width=0.90\textwidth]{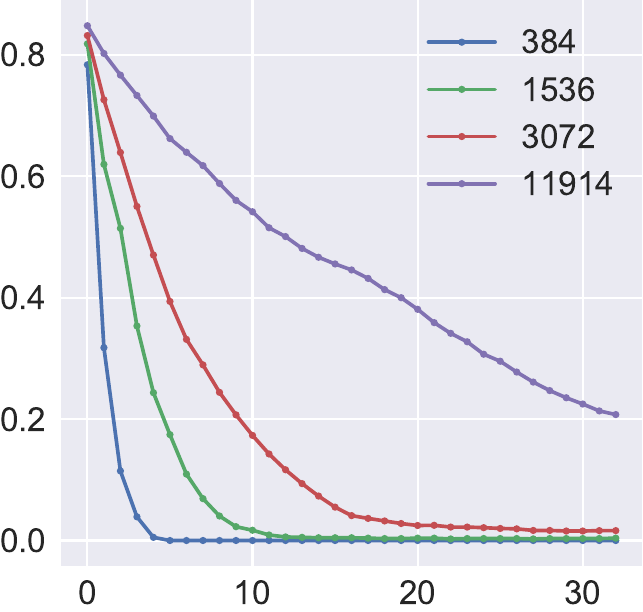}
\caption{BERT ($d=768$).}
\end{subfigure}
\end{center}
\caption{\label{fig:cv_acc}Cross-validated accuracy scores for
predicting the removed concept over INLP iterations. Each curve
corresponds to a size $n$ of the dataset.}
\end{figure}

Clearly, the behavior of the model is different from
what would have been expected if the instances were i.i.d. and $X$
independent of $Y$.
Even after just a few iterations of INLP, accuracies fall far below
the chance level.
%\todo{comment on m/n dimensionality}
This tendency is strongest for the BoW representation, which
falls to zero almost immediately.
For BERT, we see the same overall picture although INLP requires more
iterations, in particular when the dataset grows larger.
Presumably, this is because the information represented by BERT is
more difficult to express using a linear model.

\subsection{Predicted Probabilities}
\label{ss:probs}

To further illustrate the behavior of predictive models trained on
projected representations, we considered how
probabilities predicted by the models are distributed.
Figure~\ref{fig:probs} shows the distributions of predicted
probabilities for the sentiment dataset. 
We show the outputs of a
model trained on the unprocessed BERT representations and on projected
representations (10 iterations of INLP). To compare with a situation where representations
are independent of the labels, we also include
probabilities predicted by a model trained on random labels independent of the text, and we
see a clear difference between the projected and the
independent settings.

\begin{figure}[htbp]
\begin{center}
\begin{subfigure}[b]{0.15\textwidth}
\includegraphics[width=0.99\textwidth]{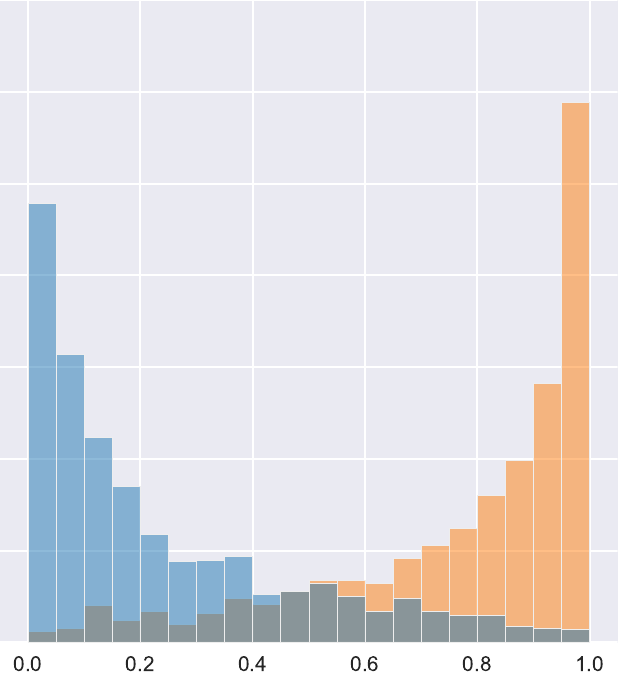}
\caption{Original.}
\end{subfigure}
\begin{subfigure}[b]{0.15\textwidth}
\includegraphics[width=0.99\textwidth]{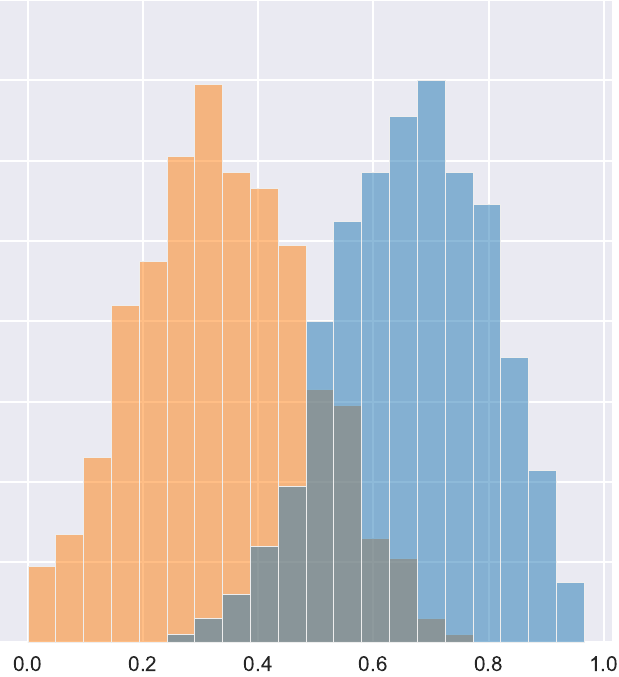}
\caption{Projected.}
\end{subfigure}
\begin{subfigure}[b]{0.15\textwidth}
\includegraphics[width=0.99\textwidth]{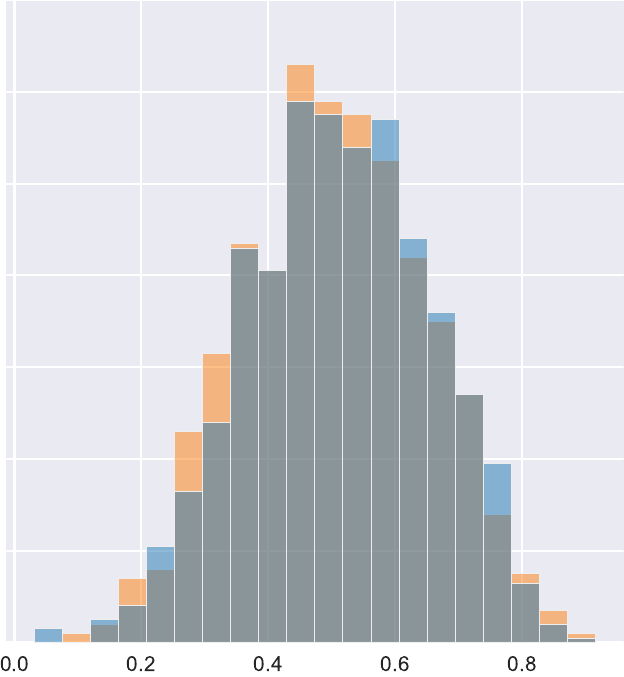}
\caption{Independent.}
\end{subfigure}

\end{center}
\caption{\label{fig:probs}Distribution of predicted probabilities for the positive (orange) and negative classes (blue).}
\end{figure}

%\todo{discuss consequences -- or not here in results?}
%TODO -- replace this for full dataset with many more iterations?

\subsection{Neighborhood Structure}
\label{ss:nnstructure}

To investigate the arrangement of instances
in the projected feature space, we carried out an experiment where we
look at how frequently the Euclidean nearest neighbors 
are of the opposite value of the target concept.
Intuitively, one would expect that when $X$ and $Y$ are unrelated and
the instances i.i.d., this
proportion should be around 0.5, while it 
would be expected to be close to 0 if there is a strong
association between $X$ and $Y$.

\begin{figure}[htbp]
\begin{center}
%\begin{subfigure}[b]{0.22\textwidth}
%x
%\caption{Random data.}
%\end{subfigure}
%\begin{subfigure}[b]{0.22\textwidth}
%x
%\caption{Reviews.}
%\end{subfigure}

\begin{subfigure}[b]{0.15\textwidth}
\includegraphics[width=0.99\textwidth]{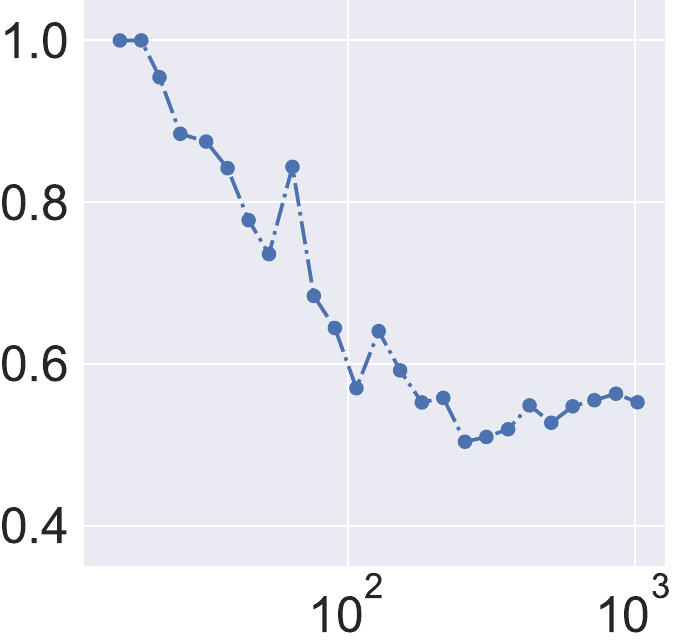}
\caption{Random.}
\end{subfigure}
\begin{subfigure}[b]{0.15\textwidth}
\includegraphics[width=0.99\textwidth]{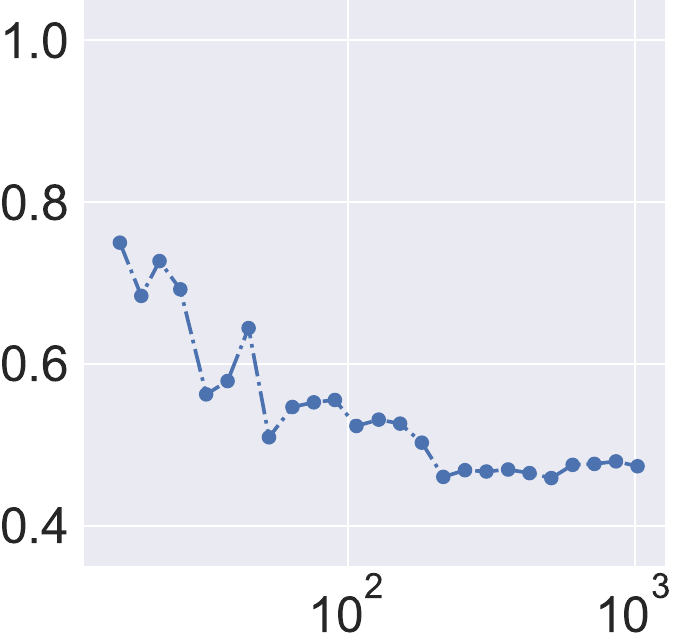}
\caption{Bag of words.}
\end{subfigure}
\begin{subfigure}[b]{0.15\textwidth}
\includegraphics[width=0.99\textwidth]{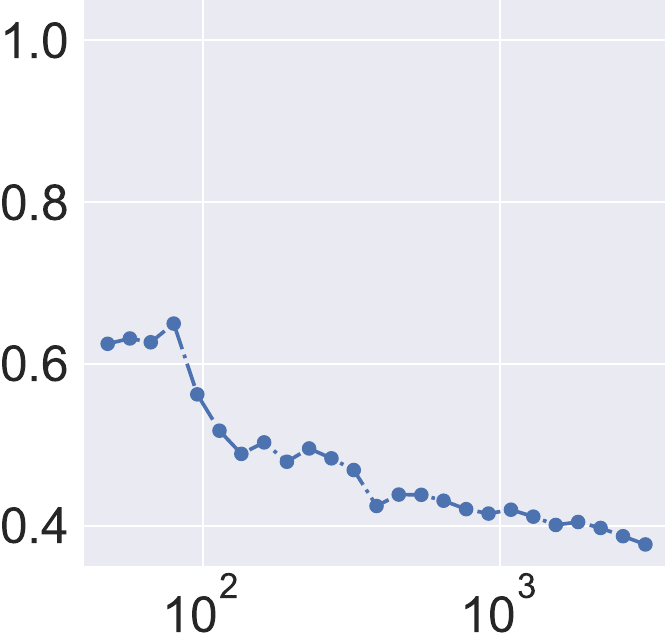}
\caption{BERT.}
\end{subfigure}

\end{center}
\caption{\label{fig:nn_prop}Proportion of instances whose nearest
neighbor is of the opposite label, for different $n$.}
\end{figure}

Figure~\ref{fig:nn_prop} shows these proportions for different data
set sizes $n$, and the tendency to place
instances near those of the opposite label is clearly visible.
This again illustrates the non-i.i.d. distribution of the projected
representations.
This tendency is most pronounced when $d \gg n$.

It is important to note that \emph{group-based} statistical measures that quantify the strength of association between $X$ and $Y$ can be misleading because the
%i.i.d. assumption is violated.
effects discussed here are discernible for \emph{individual} instances.
To illustrate, we computed the MMD %the maximum mean discrepancy %(MMD) 
\cite{gretton2012} of BERT representations between the positive and
negative groups, and we saw that the estimates steadily decrease as we
apply INLP iterations, despite the projected dataset %representations
becoming \emph{more} informative about the labels.
%

% What about mutual info estimates? Does not work well in high dimensions.
% Men skulle det funka med MINE t.ex. för BoW-256 på hela datamängden?

\subsection{Recovering the Original Grouping}
\label{ss:anticlust}

The theoretical result in \S\ref{sec:structure} and the empirical observation from \S\ref{ss:nnstructure} that instances tend to be
located close to instances of the opposite label gives an intuition
for a procedure that recovers the groups defined by the original
labeling.
Intuitively, we can partition the data points into groups selected so
that each instance is maximally dissimilar to the other instances in
the same group.
This reverses the logic of regular clustering models and has been
referred to as \emph{anti-clustering} \cite{spath1986}. For instance,
we can adapt Lloyd's algorithm for $k$-means clustering to the
anti-clustering setup, simply by changing the algorithm to assign an
instance to the cluster it is \emph{least} similar to.

We applied the \texttt{anticlust} R package \cite{papenberg2021} 
%\footnote{\url{https://cran.r-project.org/web/packages/anticlust/index.html}}
using two clusters, the diversity criterion and 100 repetitions of the 
search method by \newcite{brusco2020}. 
The clusters were then compared to the original labels of the datasets.
Figure~\ref{fig:ac_acc} shows the cluster purity scores.

%
% y-axeln är respektive accuracy.
% Jag tänker mig att en plot har INLP iterations på x-axeln.
% Differentiering av uppsättningar sker med separata linjer:
% antal rader respektive antal kolumner.
\begin{figure}[htbp]
\begin{center}
\begin{subfigure}[b]{0.15\textwidth}
\includegraphics[width=0.95\textwidth]{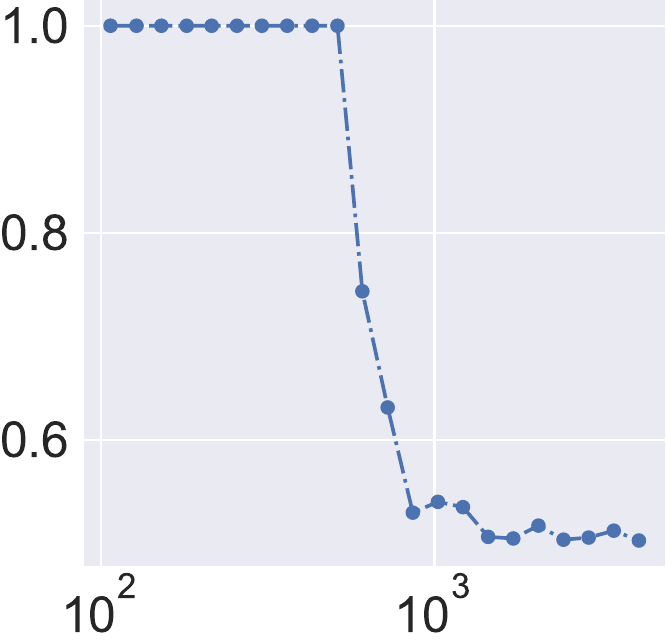}
\caption{Random.}
\end{subfigure}
\begin{subfigure}[b]{0.15\textwidth}
\includegraphics[width=0.95\textwidth]{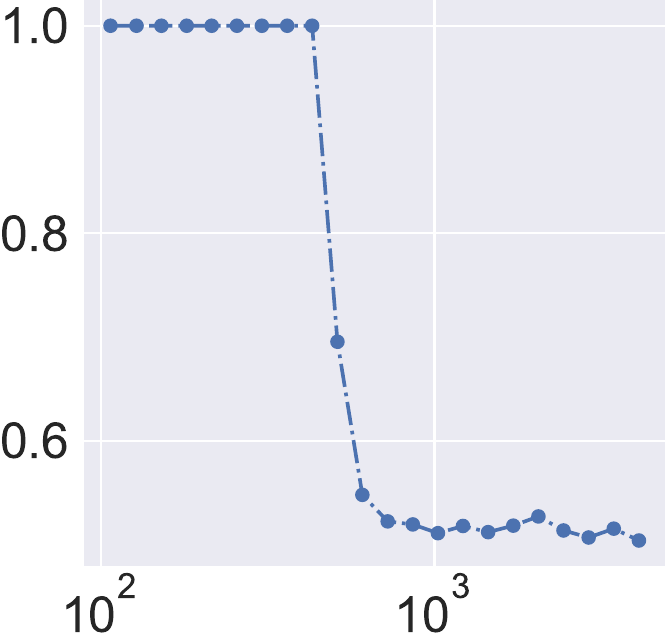}
\caption{Bag of words.}
\end{subfigure}
\begin{subfigure}[b]{0.15\textwidth}
\includegraphics[width=0.95\textwidth]{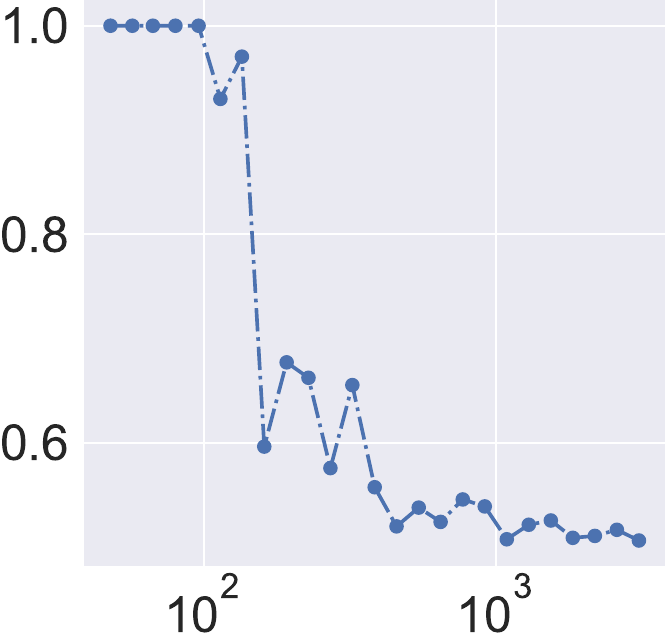}
\caption{BERT.}
\end{subfigure}
\end{center}
\caption{\label{fig:ac_acc}Cluster purity scores comparing the
original labeling to the anti-clustering result.}
\end{figure}

We observe that the anti-clustering algorithm applied to the projected
representations often perfectly reconstructs the grouping defined by
the concept we wanted to remove, in particular when \mbox{$d \gg n$.}
%
%This is 
%
As we have already argued, projection
inscribes the training labels into the data, and %we
%note that
a reconstruction is possible even if the
original dataset was %completely
random and unrelated to the training
labels.
%These purity scores are often much higher than the accuracies of
%predicting the concept in the non-projected representation
%(\S\ref{ss:cv}).

\section{Related Work}

This investigation falls into the category of work that analyzes the
behavior of concept removal methods. Most of the early discussions
focused on the pros and cons of adversarial
methods. For instance, \newcite{elazar2018} claimed that these methods
leak information; their conclusions were later challenged by \newcite{barrett2019}.

The work that is most similar in spirit to ours is arguably the investigation by \newcite{gonen2019}, which analyzed the geometric structure of word
embedding models processed by gender debiasing
methods. They argued that debiasing does not remove the
gender information, but only stores it a less obvious way, and they
showed that this information could be recovered by considering
distances in the processed space.

\section{Implications and Conclusion}

% The effects are most strongly visible in the case where the number of columns is greater than the number of rows.
% This scenario is rarer in modern NLP where it is less common to work with extremely high-dimensional representations.

% Ett stycke innan detta?

% For some applications, a more careful cross-validation can probably
% be applied.
% But if we want to process the *dataset*?

How much does it matter in practice that instances in a projected
dataset are not i.i.d.?
%
%We want to stress that p
Projection-based concept removal methods are
useful for the purpose for which they were originally
developed: transforming a dataset to make sure that a ML \emph{model} trained on
the transformed data does not rely on the target concept.
However, a naive practictioner may get the misguided
impression that the projection ``removes information'' about the concept
from the \emph{dataset} itself, when the opposite is in fact true.
Clearly, one needs to be careful if we want to use projection for the
purpose of scrubbing some signal from a dataset before distributing it.
%
%If one can afford to set aside a subset of the data in order 
We should also stress that the effects discussed here are not
problematic in case one can afford to set aside a subset of the data
reserved for the purpose of training the projection: the case we focus
on assumes that we want to use the \emph{whole} dataset.

%The strongest leakage effects demonstrated in \S\ref{ss:anticlust} are
%most visible in the case where the number of
%columns is greater than the number of rows, which is typically a less
%common use case in modern NLP than in older approaches.

A consequence of the i.i.d violation is that any statistical analysis
%or ML model
requiring strict i.i.d. assumptions is likely to be
invalid if applied to representations computed by a projection-based method.
For instance, text-based causal inference methods \cite{keith2020}
involving the text representation and the removed concept may be
affected if projection is applied: such causal inference
methods typically rely on predicted probabilities or representation
similarity, which as we have seen in \S\ref{ss:probs}
and \S\ref{ss:nnstructure} are strongly affected.
% TODO CITE ADEL PAPER
\newcite{daoud2022} and \newcite{gui2023} highlight
the problem for causal inference when the text encodes information
about a variable of interest, and our results suggest that it
could be risky to try to apply projection to remove this undesired
information.
%at least if the representations are high-dimensional.
Effects on predictions in cross-validations (\S\ref{ss:cv}) are
visible already in moderately low-dimensional settings.

% instances not iid after transformation
% but if we would consider the estimated distributions (e.g. MMD), they would look the same

% future work: what other concept removal methods have this property?
%In future work, we would like to explore a wider range of concept
%removal methods to see to what extent they are affected by
%the same phenomena we have highlighted in these investigations.

% 4 pages end here.
\section{Limitations}

There are a number of ways in which this work could be put on
firmer ground theoretically.
In \S\ref{sec:structure}, we limited the theoretical analysis to MP (or equivalently, INLP
based on a nearest centroid classifier), and in future work we would
like to find a more general formal justification for why the
adversarial arrangement emerges.
%nearest-centroid classifier is used to find the normal vector for the
%projection;
%In our experiments, following most previous work, we instead used
%logistic regression models for this purpose. We currently do not have a
%formal justification for why the adversarial arrangement emerges when
%using this type of projection.
In the empirical section, we would also like to take a more general
approach in the future and investigate additional concept removal
methods, such as more recent projection-based methods as well as
adversarial representation learning methods.

Furthermore, we do not have a clear understanding of the role played by
the dimensionality $d$ in relation to the dataset size
$n$.
%In the intuitive explanation given in \S\ref{sec:structure}, it
%seems that it should be easier for the projection to encode the labels
%in the transformed data in a high-dimensional setting, and this is
The experiments (\S\ref{ss:nnstructure}
and \S\ref{ss:anticlust}) indicate that that such effects play a role, but this
is currently not taken into account in the theoretical analysis.

\section{Ethical Discussion}

Whether the behaviors investigated here matter in practice depend on
the application, and as discussed above, the consequences are likely
to be limited if the only purpose of the processed representations is
for training a model.
In other cases, in particular when the intention is for the projected
dataset to be distributed, the effects may be more problematic.
For instance, if projection is applied to a set of word embeddings in
order to make them invariant to a demographic attribute, we may accidentally
%arrange the processed embeddings in a gender-related geometry that was
%not even present in the first place,
encode information about the attribute into the 
embedding geometry, so that it can later be decoded from %the
representations.

% useless for privacy -- although be clear that this has not been claimed
Furthermore, the fact that in many cases the original groups can be
reconstructed from the projected data (\S\ref{ss:anticlust}), even if
the original dataset did not encode any information about the target
concept, shows that projection-based methods should not be viewed as
privacy-preserving \cite{coavoux2018}. To be clear, the inventors of
the methods we have considered did not claim that they are intended
to ensure privacy,\footnote{In contrast, \newcite{xu2017} explicitly
considered projection for privacy, but we have not investigated their method.} but again it is important for users to understand
that projection is not equivalent to information removal in a dataset.

\section{Acknowledgements}

The results presented here are a by-product of discussions with Adel Daoud.
This research was supported by the projects \emph{Interpreting and Grounding Pre-trained Representations for NLP} and \emph{Representation Learning for Conversational AI}, both under the Wallenberg  AI,  Autonomous  Systems  and  Software Program (WASP) funded by the  Knut  and  Alice Wallenberg Foundation, and the project \emph{Countering Bias in AI Methods in the Social Sciences}
%This research was funded by the Wallenberg AI, Autonomous Systems and Software Program (WASP) funded by the Knut and Alice Wallenberg Foundation, and
under the 
Wallenberg AI, Autonomous Systems and Software Program -- Humanity and Society (WASP-HS),
funded by the Marianne and Marcus Wallenberg Foundation and the Marcus and Amalia Wallenberg Foundation.

%Please note that extra space is allowed after the 8th page (4th page
%for short papers) for an ethics/broader impact statement and a
%discussion of limitations. At submission time, this means that if you
%need extra space for these sections, it should be
%placed after the conclusion so that it is possible to rapidly check
%that the rest of the paper still fits in 8 pages (4 pages for short
%papers). Ethical considerations sections, limitations, acknowledgements, and
%references do not count against these limits. For camera-ready
%versions 9 pages of content will be allowed for long (5 for short)
%papers.

%
%\nocite{*}
\section{Bibliographical References}\label{reference}
%\label{main:ref}

\bibliographystyle{lrec-coling2024-natbib}
\bibliography{paper}

\end{document}